\documentclass{sig-alternate-05-2015}
\pdfoutput=1
\usepackage[noend]{algorithmic}
\usepackage{algorithm}
\usepackage{amstext,amssymb,amsmath}
\usepackage{graphicx}
\usepackage{xcolor}
\usepackage{tcolorbox}
\usepackage{xspace}
\usepackage{hyperref}
\usepackage{fancyvrb}
\usepackage{soul}
\usepackage{eso-pic}

\renewcommand{\algorithmicrequire}{\textbf{Input:}}
\renewcommand{\algorithmicensure}{\textbf{Output:}}
\newcommand{\M}{\mathcal{M}}  % mechanism
\newcommand{\R}{\mathbb{R}}
\newcommand{\D}{d}  % dataset / database
\newcommand{\Domain}{\mathcal{D}}
\newcommand{\Range}{\mathcal{R}}
\newcommand{\aux}{\textsf{aux}}
\newcommand{\eqdef}{\stackrel{\Delta}{=}}
\newcommand{\E}{\mathbb{E}}
\newcommand{\leader}[1]{\medskip\noindent\textbf{#1}}
\newcommand{\eqr}[1]{Eq.~\eqref{#1}}

\newcommand{\btheta}{{\mathbb \theta}}
\newcommand{\calL}{\ensuremath{\mathcal L}}
\newcommand{\calN}{\ensuremath{\mathcal N}}
\newcommand{\g}{\ensuremath{\mathbf g}}
\newcommand{\Id}{\ensuremath{\mathbf I}}

\newcommand{\eps}{\varepsilon}
\newcommand{\x}{\ensuremath{\mathbf x}}
\newcommand{\full}[2]{#1}

\newtheorem{thm}{Theorem}
\newtheorem{lem}[thm]{Lemma}
\newdef{definition}{Definition}

\newcommand{\mynote}[3]{[{\color{#2}{\textbf{#1} \emph{#3}}}]}

\newcommand{\ktnote}[1]{\mynote{KT}{red}{#1}}

\def\logmgfa{moments\xspace}
\def\Logmgfa{Moments\xspace}

\newcommand{\mycomment}[1]{}

\begin{document}
\full{
\AddToShipoutPicture*{\AtPageUpperLeft{\small
		{\raisebox{-1cm}{\hspace{1in}\parbox{\textwidth}{A preliminary version of this paper appears in the proceedings of the \emph{23rd ACM Conference on Computer and Communications Security (CCS 2016)}. This is a full version.}}}}}}{}
	
\numberofauthors{6} % Actually, 7 but setting it to 7 creates an unncessary "Additional Authors" section.
\author{
	\alignauthor
	Mart\'in Abadi\thanks{Google.}
	\alignauthor
	 Andy Chu\footnotemark[1]
	 \alignauthor
	 Ian Goodfellow\thanks{OpenAI. Work done while at Google.}
	 \and
     \alignauthor	  
	 H.~Brendan McMahan\footnotemark[1]
	 \alignauthor
	 Ilya Mironov\footnotemark[1]
	 \alignauthor
	 Kunal Talwar\footnotemark[1]
	 \and
	 \alignauthor
	 Li Zhang\footnotemark[1]}
	
\additionalauthors{}

\CopyrightYear{2016} 
\setcopyright{rightsretained} 
\conferenceinfo{CCS'16}{October 24-28, 2016, Vienna, Austria} 
\isbn{978-1-4503-4139-4/16/10}
\doi{http://dx.doi.org/10.1145/2976749.2978318}

\clubpenalty=10000 
\widowpenalty = 10000

\title{Deep Learning with Differential Privacy\full{\\ \medskip \fontsize{14}{14}\textsf{\today}}{}}

\maketitle

\abstract{Machine learning techniques based on neural networks are achieving remarkable results in a wide variety of domains. Often, the training of models requires large,
representative datasets, which may be crowdsourced and contain
sensitive information. The models should not expose private
information in these datasets.  Addressing this goal, we develop new
algorithmic techniques for learning and a refined analysis of privacy
costs within the framework of differential privacy. Our implementation
and experiments demonstrate that we can train deep neural networks
with non-convex objectives, under a modest privacy budget, and at a
manageable cost in software complexity, training efficiency, and model
quality.}

\section{Introduction} % Ilya
%% -*- TeX-master: "main" -*-

Recent progress in neural networks has led to impressive successes in
a wide range of applications, including image classification, language
representation, move selection for Go, and many more
(e.g.,~\cite{Szegedy-et-al-CVPR2015,he2015delving,Vinyals-et-al-Grammar,maddison2015go,GoNature16}). These advances are enabled, in part, by the
availability of large and representative datasets for training 
neural networks. These datasets are often crowdsourced, and may
contain sensitive information.  Their use requires techniques that
meet the demands of the applications while offering principled and rigorous privacy guarantees.

In this paper, we combine
state-of-the-art machine learning methods with advanced
privacy-preserving mechanisms, training neural networks within a
modest (``single-digit'') privacy budget.  
We treat models with non-convex objectives, several layers, and tens of thousands to millions of parameters. (In contrast,
previous work obtains strong results on convex models with smaller numbers
of parameters, or treats complex neural networks but with a large
privacy loss.)
For this purpose, we
develop new algorithmic techniques, a refined analysis of privacy
costs within the framework of differential privacy, and careful
implementation strategies:
\begin{enumerate}

\item We demonstrate that, by tracking detailed information (higher
  moments) of the privacy loss, we can obtain much tighter estimates on
  the overall privacy loss, both asymptotically and empirically.

\item We improve the computational efficiency of differentially
  private training by introducing new techniques. These techniques
  include efficient algorithms for computing gradients for individual
  training examples, subdividing tasks into smaller batches to reduce memory footprint, and
  applying differentially private principal projection at the input
  layer.

% This suggests that we can build on the wealth of recent work on the deep network architectures.

\item We build on the machine learning framework
  TensorFlow~\cite{tensorflow2015-whitepaper} for training models with differential privacy.  
We evaluate our approach on two standard image classification
tasks, MNIST and CIFAR-10. We chose these two tasks because they are based on public data\-sets and have a long record of serving as benchmarks in machine learning.
Our experience indicates that privacy protection for deep neural networks can be
  achieved at a modest cost in software complexity, training
  efficiency, and model quality.
\end{enumerate}

Machine learning systems often comprise elements that contribute to
protecting their training data. In particular, regularization
techniques, which aim to avoid overfitting to the examples used for
training, may hide details of those examples. On the other hand,
explaining the internal representations in deep neural networks is
notoriously difficult, and their large capacity entails that these
representations may potentially encode fine details of at least some
of the training data. In some cases, a determined adversary may be
able to extract parts of the training data. For example, Fredrikson et
al.~demonstrated a model-inversion attack that recovers images
from a facial recognition system~\cite{FJR15}. 

While the model-inversion attack requires only ``black-box'' access to
a trained model (that is, interaction with the model via inputs and
outputs), we consider adversaries with additional capabilities, much
like Shokri and Shmatikov~\cite{ShokriShmatikov15}.  
Our approach offers protection against a strong adversary with full knowledge of the
training mechanism and access to the model's parameters.  
This protection is attractive, in particular, for applications of
machine learning on mobile phones, tablets, and other devices. Storing
models on-device enables power-efficient, low-latency inference, and
may contribute to privacy since inference does not require
communicating user data to a central server; on the other hand, we
must assume that the model parameters themselves may be exposed to
hostile inspection.
Furthermore, when we are concerned with preserving the privacy
of one record in the training data, we allow for the possibility
that the adversary controls some or even all of the rest of the training data.
In practice, this possibility cannot always be excluded, for example 
when the data is crowdsourced.

The next section reviews background on deep learning and on differential privacy.
Sections~\ref{sec:approach} and \ref{sec:impl} explain our approach and implementation. Section~\ref{sec:results} describes our experimental results.
Section~\ref{sec:related} discusses related work, and Section~\ref{sec:conclusions} concludes.
Deferred proofs appear in the \full{Appendix}{full version of the paper~\protect{\cite{DL-DP-arxiv}}}.

\section{Background} % Ilya, Brendan
In this section we briefly recall the definition of differential privacy, introduce the Gaussian mechanism and composition theorems, and overview basic principles of deep learning.

\subsection{Differential Privacy}
%% -*- TeX-master: "main" -*-

Differential privacy~\cite{DMNS06,Dwork-CACM,DworkRoth14} constitutes a
strong standard for privacy guarantees for algorithms on aggregate
data\-bases. It is defined in terms of the
application-specific concept of adjacent databases. In our
experiments, for instance, each training dataset is a set of
image-label pairs; we say that two of these sets are adjacent if they
differ in a single entry, that is, if one image-label pair is
present in one set and absent in the other.

\begin{definition}A randomized mechanism $\M\colon \Domain\rightarrow\Range$ with domain $\Domain$ and range $\mathcal{R}$ 
	satisfies $(\eps,\delta)$-differential privacy if for any two adjacent inputs $\D,\D'\in \Domain$ and for any subset of outputs $S\subseteq\Range$ it holds that
	\[
	\Pr[\M(\D)\in S]\leq e^{\eps}\Pr[\M(\D')\in S]+\delta.
	\]
\end{definition}
The original definition of $\eps$-differential privacy does not 
include the additive term $\delta$.  We use the variant
introduced by Dwork et al.~\cite{ODO}, which allows for the
possibility that plain $\eps$-differential privacy is broken with
 probability~$\delta$ (which is preferably smaller than $1/|d|$).

Differential privacy has several properties that
make it particularly useful in applications such as ours:
composability, group privacy, and robustness to auxiliary
information. Composability enables modular design of mechanisms: if
all the components of a mechanism are differentially private, then so
is their composition. Group privacy implies graceful degradation of
privacy guarantees if datasets contain correlated inputs, such as the
ones contributed by the same individual. Robustness to auxiliary
information means that privacy guarantees are not affected by any side
information available to the adversary.

A common paradigm for approximating a deterministic real-valued
function $f\colon \Domain\rightarrow\mathbb{R}$ with a
differentially private mechanism is via additive noise calibrated to
$f$'s \emph{sensitivity} $S_f$, which is defined as the maximum of the 
absolute distance 
$|f(\D)-f(\D')|$ where $\D$ and $\D'$ are adjacent inputs.
(The restriction to a real-valued function is intended to simplify this
review, but is not essential.)
For instance, the Gaussian noise mechanism is defined by
\[
\M(\D)\eqdef f(\D)+\calN(0, S_f^2\cdot \sigma^2),
\]
where $\calN(0, S_f^2\cdot \sigma^2)$ is the normal (Gaussian) distribution with mean 0 and standard deviation $S_f \sigma$.
A single application of the Gaussian mechanism to function $f$ of
sensitivity $S_f$ satisfies $(\eps, \delta)$-differential privacy if
$\delta\geq \frac45 \exp(-(\sigma\eps)^2/2)$ and $\eps<1$~\cite[Theorem
3.22]{DworkRoth14}. Note that this analysis of the mechanism can be
applied \emph{post hoc}, and, in particular, that there are infinitely
many $(\eps,\delta)$ pairs that satisfy this condition.
%MA: it is *not* a curve of such values, because the condition is an inequality

Differential privacy for repeated applications of additive-noise 
mechanisms follows from the basic composition
theorem~\cite{ODO,DworkLei09}, or from advanced composition theorems
and their
refinements~\cite{DRV10-boosting,KOV15,DworkRothR16,BunS16}. The task
of keeping track of the accumulated privacy loss in the course of
execution of a composite mechanism, and enforcing the applicable
privacy policy, can be performed by the \emph{privacy accountant},
introduced by McSherry~\cite{PINQ}.

The basic blueprint for designing a differentially private
additive-noise mechanism that implements a given functionality
consists of the following steps: approximating the functionality by a
sequential composition of bounded-sensitivity functions; choosing
parameters of additive noise; and performing privacy analysis of the
resulting mechanism. We follow this approach in
Section~\ref{sec:approach}.

\subsection{Deep Learning} % Brendan
Deep neural networks, which are remarkably effective for many machine
learning tasks, define parameterized functions from inputs
to outputs as compositions of many layers of basic building blocks, such as affine transformations and simple nonlinear functions. Commonly used examples of the latter are sigmoids and rectified linear units (ReLUs). By varying parameters of these blocks, we can ``train'' such a parameterized function with the goal of fitting any given finite set of input/output examples.

More precisely, we define a loss
function $\calL$ that represents the penalty for mismatching the
training data.  The loss $\calL(\btheta)$ on parameters $\btheta$ is
the average of the loss over the training examples $\{x_1, \ldots,
x_N\}$, so $\calL(\btheta) = \frac{1}{N}\sum_i \calL(\btheta, x_i)$.
Training consists in finding $\btheta$ that yields an acceptably small
loss, hopefully the smallest loss (though in practice we seldom expect to reach
an exact global minimum).

For complex networks, the loss function $\calL$ is usually non-convex
and difficult to minimize. In practice, the minimization is often done
by the mini-batch stochastic gradient descent (SGD) algorithm. In this
algorithm, at each step, one forms a batch $B$ of random examples and
computes $\g_B = 1/|B| \sum_{x\in B} \nabla_\btheta\calL(\btheta, x)$
as an estimation to the gradient $\nabla_\btheta\calL(\btheta)$. Then
$\btheta$ is updated following the gradient direction $-\g_B$ towards
a local minimum.

Several systems have been built to support the definition of neural
networks, to enable efficient training, and then to perform efficient
inference (execution for fixed parameters)~\cite{lua,torch7,tensorflow2015-whitepaper}.  We base our work on
TensorFlow, an open-source dataflow engine released by
Google~\cite{tensorflow2015-whitepaper}.  TensorFlow allows the programmer to define
large computation graphs from basic operators, and to distribute their
execution across a heterogeneous distributed system.  TensorFlow
automates the creation of the computation graphs for gradients; it
also makes it easy to batch computation.

%MA: I started to rewrite this paragraph, which needed work, and stopped
%    because I don't see what it is trying to achieve and why it is needed
%% So, in TensorFlow, the programmer would 
%% define the neural network and the loss function
%% $\calL(\btheta,x)$, call the symbolic differentiation function
%% $\texttt{gradient}(\calL,\btheta)$ to get the computation graph for
%% $\nabla_\btheta \calL(\btheta, x)$, and form the graph for the
%% gradient descent step. Then the graph can be executed to perform the
%% minimization of loss, or training, process.

%MA: I would like this to be elsewhere
%% Since the gradient computation depends on the training data, 
%% our approach for privacy-preserving training protecting the
%% privacy of the gradient. 

%MA: The following does not seem necessary or particularly useful.
%% Hence the availability of the gradient in TensorFlow suits our goal well. 
%% While we have done our implementation
%% on TensorFlow, it should be trivial to port our work to the other
%% similar machine learning platforms.

% SGD, batches, TensorFlow

\section{Our approach}\label{sec:approach} % Kunal, followed by Brendan, Li
%% -*- TeX-master: "main" -*-

This section describes the main components of our approach toward differentially
private training of neural networks: a differentially private stochastic gradient descent (SGD) algorithm, the moments accountant, and hyperparameter tuning.

\subsection{Differentially Private SGD Algorithm}

One might attempt to protect the privacy of training data by
working only on the final parameters that result from the training
process, treating this process as a black box.  Unfortunately, in
general, one may not have a useful, tight characterization of the
dependence of these parameters on the training data; adding overly conservative
noise to the parameters, where the noise is selected according to the worst-case analysis, would destroy the utility of the learned model.
Therefore, we prefer a more sophisticated approach in which
we aim to control the influence of the training data during the
training process, specifically in the SGD computation.
This approach has been followed 
in previous works (e.g.,~\cite{SongCS13,BassilyTS14}); we make several
modifications and extensions, in particular in our privacy accounting.

Algorithm~\ref{alg:privsgd} outlines our basic
method for training a model with parameters
$\btheta$ by minimizing the empirical loss function
$\calL(\btheta)$. At each step of the SGD, we compute the gradient
$\nabla_\btheta\calL(\btheta, x_i)$ for a random subset of examples,
clip the $\ell_2$ norm of each gradient, compute the average, add
noise in order to protect privacy, and take a step in the opposite direction of this
average noisy gradient. At the end, in addition to outputting the model, we will also need to compute the privacy loss of the mechanism based on the information maintained by the privacy accountant. Next we describe in more detail each component of this algorithm and our refinements.

\begin{algorithm}[htb]
	\caption{Differentially private SGD (Outline)}\label{alg:privsgd}
	\begin{algorithmic}
	\REQUIRE Examples $\{x_1,\ldots,x_N\}$, loss function $\calL(\btheta)=\frac{1}{N}\sum_i \calL(\btheta, x_i)$. Parameters: learning rate $\eta_t$, noise scale $\sigma$, group size $L$, gradient norm bound $C$. 
%\bmnote{Give a formula for choosing $\sigma_t$. Also: why two Input: (\REQUIRE) lines?}
%LZ: we are adopting the strategy of choosing sigma and then using privacy
%accountant to accumulate the privacy loss. Maybe we should make that more
%explicit in Privacy Accountant section?
		\STATE {\bf Initialize} $\btheta_0$ randomly
		\FOR{$t \in [T]$}
		\STATE {Take a random sample $L_t$ with sampling probability $L/N$}
		\STATE {\bf Compute gradient}
		\STATE {For each $i\in L_t$, compute $\g_t(x_i) \gets \nabla_{\btheta_t} \calL(\btheta_t, x_i)$}		
		\STATE {\bf Clip gradient}
		\STATE {$\bar{\g}_t(x_i) \gets \g_t(x_i) / \max\big(1, \frac{\|\g_t(x_i)\|_2}{C}\big)$}
		\STATE {\bf Add noise}
		\STATE {$\tilde{\g}_t \gets \frac{1}{L}\left( \sum_i \bar{\g}_t(x_i) + \mathcal{N}(0, \sigma^2 C^2 \Id)\right)$}
		\STATE {\bf Descent}
		\STATE { $\btheta_{t+1} \gets \btheta_{t} - \eta_t \tilde{\g}_t$}
		\ENDFOR
		\STATE {\bf Output} $\btheta_T$ and compute the overall privacy cost $(\eps, \delta)$ using a privacy accounting method.
%\bmnote{Ref to Eqs or pseudocode for computing privacy cost.}
%LZ: we will spend most of time discussing how privacy is computed. Since
%the algorithm description is an outline, maybe we can postpone it?
	\end{algorithmic}
\end{algorithm}

\leader{Norm clipping:} Proving the differential privacy guarantee of
Algorithm~\ref{alg:privsgd} requires bounding the influence of each
individual example on $\tilde{\g}_t$. Since there is no {\em a priori}
bound on the size of the gradients, we {\em clip} each gradient in
$\ell_2$ norm; i.e., the gradient vector $\g$ is replaced by $\g/
\max\big(1, \frac{\|\g\|_2}{C}\big)$, for a clipping threshold $C$. This clipping
ensures that if $\|\g\|_2 \leq C$, then $\g$ is preserved, whereas if
$\|\g\|_2 > C$, it gets scaled down to be of norm $C$. We remark that
gradient clipping of this form is a popular ingredient of SGD for deep
networks for non-privacy reasons, though in that setting it usually
suffices to clip after averaging.

\leader{Per-layer and time-dependent parameters:} The pseudocode for Algorithm~\ref{alg:privsgd}
groups all the parameters into a single input $\btheta$ of the loss function $\calL(\cdot)$.  For multi-layer neural networks, we consider each
layer separately, which allows setting different clipping thresholds $C$ and noise scales $\sigma$ for different layers. Additionally, the clipping and noise parameters may vary with the number
of training steps $t$. In results presented in Section~\ref{sec:results} we use constant settings for $C$ and $\sigma$.
%\bmnote{I think both for purposes of noise scales and clipping, in a full version it might be better to handle layers explicitly in Algorithm 1, in the spirit of giving pseudocode that reflects what we actually do as accurately as possible (as well as to make it as easy as possible for someone else to replicate our experiments.} 
%LZ: Agreed. We will leave it for now.
%ß

\leader{Lots:} Like the ordinary SGD algorithm, Algorithm~\ref{alg:privsgd}
estimates the gradient of $\calL$ by computing the gradient of the
loss on a group of examples and taking the average. This average provides an
unbiased estimator, the variance of which decreases quickly with the
size of the group. We call such a group a {\em lot}, to distinguish it
from the computational grouping that is commonly called a {\em batch}. 
In order to limit memory consumption, we may set the batch size much
smaller than the lot size $L$, which is a parameter of the algorithm.
We perform the computation in batches, then group several batches into a lot for adding noise. 
In practice, for efficiency, the construction of batches and lots is done by randomly permuting the
examples and then partitioning them into groups of the appropriate sizes.
%MA: streamline
% This is merely for computation efficiency and does not affect our analysis. 
For ease of analysis, however, we assume that each lot is formed by independently picking each example with probability $q=L/N$, where $N$ is the size of the input dataset.

As is common in the literature, we normalize the running time of a training algorithm by
expressing it as the number of \emph{epochs}, where each epoch is the (expected) 
number of batches required to process $N$ examples. In our notation, an epoch
consists of $N/L$ lots.

\leader{Privacy accounting:} For differentially private SGD, an important
issue is computing the overall privacy cost of the training. The composability
of differential privacy allows us to implement an ``accountant'' procedure that
computes the privacy cost at each access to the training data,
and accumulates this cost as the training progresses.  Each step of training typically
requires gradients at multiple layers, and the accountant accumulates the cost that corresponds to
all of them.

\leader{\Logmgfa accountant:} Much research has been devoted to
studying the privacy loss for a particular noise distribution as well
as the composition of privacy losses.  For the Gaussian noise that we use, if
we choose $\sigma$ in Algorithm~\ref{alg:privsgd} to be
$\sqrt{2\log\frac{1.25}{\delta}}/\eps$, then by standard
arguments~\cite{DworkRoth14} each step is $(\eps,\delta)$-differentially private with
respect to the lot. Since the lot itself is a random sample from the
database, the privacy amplification theorem~\cite{KasiviswanathanLNRS11,BeimelBKN14} implies
that each step is $(O(q\eps), q\delta)$-differentially private with
respect to the full database where $q=L/N$ is the sampling ratio per
lot and $\eps \leq 1$. The result in the literature that yields the best overall bound
is the strong composition theorem~\cite{DRV10-boosting}.

%%would then
%%imply that Algorithm~\ref{alg:privsgd} is $(\eps',\delta')$-differentially
%%private for
%%\begin{equation}\label{eq:strongcompos}
%%\eps' = q^2\eps^2T + q\eps\sqrt{2T\log (1/\delta_1)}\,,\quad \delta'=\delta%%_1 + Tq\delta\,,
%%\end{equation}
%%for any $\delta_1>0$. 

However, the strong composition theorem can be loose, and does not take into account the
particular noise distribution under consideration.
In our work, we invent a stronger accounting method, which
we call the \logmgfa accountant. It allows us to prove that
Algorithm~\ref{alg:privsgd} is $(O(q\eps\sqrt{T}), \delta)$-differentially private for appropriately chosen settings of the noise scale and the clipping threshold.
Compared to what one would obtain by the strong composition theorem, our bound is tighter in two
ways: it saves a $\sqrt{\log(1/\delta)}$ factor in the $\eps$
part and a $Tq$ factor in the $\delta$ part. Since we expect $\delta$ to be small and $T \gg 1/q$ (i.e., each example is examined multiple times), the saving provided by our bound is quite significant. This result is one of our main contributions.
\begin{thm}\label{thm:main}
There exist constants $c_1$ and $c_2$ so that given the sampling probability $q=L/N$ and the number of steps $T$, for any $\eps < c_1 q^2T$,  Algorithm~\ref{alg:privsgd} is $(\eps,\delta)$-differentially private for any $\delta>0$ if we choose 
\begin{align*}
\sigma \geq c_2\frac{q \sqrt{T \log(1/\delta)}}{\eps}\,.
\end{align*}
\end{thm}

If we use the strong composition theorem, we will then need to choose
$\sigma=\Omega(q\sqrt{T\log(1/\delta)\log(T/\delta)}/\eps)$. Note that
we save a factor of $\sqrt{\log(T/\delta)}$ in our asymptotic
bound. The \logmgfa accountant is beneficial in theory, as this result
indicates, and also in practice, as can be seen from
Figure~\ref{fig:epoch-eps} in Section~\ref{sec:impl}.  For example,
with $L=0.01N$, $\sigma=4$, $\delta=10^{-5}$, and $T=10000$, we have
$\eps\approx 1.26$ using the \logmgfa accountant. As a comparison, we
would get a much larger $\eps\approx 9.34$ using the strong
composition theorem.

%% -*- TeX-master: "main" -*-

\newcommand{\outcome}{o}
\subsection{The \Logmgfa Accountant: Details}

The \logmgfa accountant keeps track of a bound on the moments of the
privacy loss random variable (defined below in \eqr{eq:privacyloss}).
It generalizes the standard approach of tracking $(\eps,\delta)$ and
using the strong composition theorem. While such an improvement was
known previously for composing Gaussian mechanisms, we show that it
applies also for composing Gaussian mechanisms with random sampling
and can provide much tighter estimate of the privacy loss of
Algorithm~\ref{alg:privsgd}.

Privacy loss is a random variable dependent on the random noise added
to the algorithm.  That a mechanism $\cal{M}$ is
$(\eps,\delta)$-differentially private is equivalent to a certain tail
bound on $\cal{M}$'s privacy loss random variable.  While the tail
bound is very useful information on a distribution, composing directly
from it can result in quite loose bounds.  We instead compute the log
moments of the privacy loss random variable, which compose
linearly. We then use the moments bound, together with the
standard Markov inequality, to obtain the tail bound, that is the
privacy loss in the sense of differential privacy.

More specifically, for neighboring databases $d, d' \in \Domain^n$, a
mechanism $\M$, auxiliary input \textsf{aux}, and an outcome $\outcome \in \Range$, define the
privacy loss at $\outcome$ as
\begin{equation}\label{eq:privacyloss}
  c(\outcome; \M,  \textsf{aux}, d, d') \eqdef \log \frac{\Pr[\M( \textsf{aux}, d) = \outcome]}{\Pr[\M( \textsf{aux}, d') = \outcome]}.
\end{equation}
%%LZ: We no longer omit arguments for clarity.
%%When $\M,  \textsf{aux}, d, d'$ are obvious from context, we will omit those and
%%simply write $c(\outcome)$.

A common design pattern, which we use extensively in the paper, is to update the state by sequentially applying differentially private mechanisms. This is an instance of \emph{adaptive composition}, which we model by letting the auxiliary input of the $k^\textrm{th}$ mechanism $\M_k$ be the output of all the previous mechanisms.

For a given mechanism $\M$, we define the $\lambda^\textrm{th}$ moment $\alpha_\M(\lambda;\aux, d, d')$ as the log of the \emph{moment generating function} evaluated at the value~$\lambda$:
\begin{multline}\label{eq:logmgf}
\alpha_\M(\lambda;\aux,d,d') \eqdef \\
\log \E_{\outcome \sim \M(\textsf{aux},d)}[\exp(\lambda c(\outcome; \M, \textsf{aux}, d, d'))].
\end{multline}
%%\bmnote{I think we should keep the dependence on $d'$ explicit here. This is a dependency that \emph{is} here, but \emph{cannot} be present in any claim about the privacy of a general mechanism, and so tracking where we eliminate that dependence is worth being explicit about. I'd propose adding something like this here, making the corresponding updates below.}
%%LZ: I agree. Made it explicit.

In order to prove privacy guarantees of a mechanism, it is useful to bound all possible $\alpha_\M(\lambda; \aux, d, d')$. We define
\[\alpha_\M(\lambda) \eqdef \max_{\aux, d, d'} \alpha_\M(\lambda; \aux, d, d')\,,\]
where the maximum is taken over all possible $\aux$ and all the neighboring databases $d,d'$.

We state the properties of $\alpha$ that we use for the
\logmgfa accountant. 
\begin{thm}\label{thm:property}
Let $\alpha_\M(\lambda)$ defined as above. Then

\begin{enumerate}
\item\label{thm:partone} \textbf{[Composability]}
Suppose that a mechanism $\M$ consists of a sequence of adaptive mechanisms $\M_1, \ldots, \M_k$ where $\M_i\colon \prod_{j=1}^{i-1}\Range_j\times \Domain \to\Range_i$. Then, for any $\lambda$
\[\alpha_\M(\lambda) \leq \sum_{i=1}^k \alpha_{\M_i}(\lambda)\,.\]

\item\label{thm:parttwo} \textbf{[Tail bound]}
For any $\eps>0$, the mechanism $\M$ is $(\eps, \delta)$-differentially private for
\[\delta=\min_{\lambda} \exp(\alpha_\M(\lambda) -\lambda \eps)\,.\]
\end{enumerate}
\end{thm}

In particular, Theorem~\ref{thm:property}.\ref{thm:partone} holds when the mechanisms themselves are chosen based on the (public) output of the previous mechanisms. 

By Theorem~\ref{thm:property}, it suffices to compute, or bound, $\alpha_{\M_i}(\lambda)$ at
each step and sum them to bound the moments of the mechanism
overall. We can then use the tail bound to convert the moments bound
to the $(\eps,\delta)$-differential privacy guarantee. 

The main challenge that remains is to bound the value $\alpha_{\M_t}(\lambda)$ for each
step. In the case of a Gaussian mechanism with random sampling, it suffices to estimate the following moments.  Let $\mu_0$ denote
the probability density function (pdf) of $\calN(0, \sigma^2)$, and $\mu_1$ denote the pdf of $\calN(1, \sigma^2)$. Let $\mu$ be the mixture of two Gaussians
$\mu=(1-q)\mu_0+q\mu_1$.  Then we need to compute
$\alpha(\lambda)=\log\max(E_1, E_2)$ where
\begin{align}
E_1 & = \E_{z \sim \mu_0} [(\mu_0(z) / \mu(z))^\lambda]\,,\label{eq:logmgf1}\\
E_2 & = \E_{z \sim \mu_{\hphantom{0}}} [ (\mu(z) / \mu_0(z))^\lambda]\,.\label{eq:logmgf2}
\end{align}

In the implementation of the \logmgfa accountant, we carry out
numerical integration to compute $\alpha(\lambda)$.  In addition, we
can show the asymptotic bound
\begin{align*}
\alpha(\lambda) &\leq q^2\lambda(\lambda+1)/(1-q)\sigma^2 + O(q^3/\sigma^3)\,.
\end{align*}

Together with Theorem~\ref{thm:property}, the above bound implies our main Theorem~\ref{thm:main}. The details can be found in the \full{Appendix}{full version of the paper~\protect\cite{DL-DP-arxiv}}.

\subsection{Hyperparameter Tuning}\label{sec:hyper}

We identify characteristics of models relevant for privacy and,
specifically, hyperparameters that we can tune in order to balance
privacy, accuracy, and performance. In particular, through
experiments, we observe that model accuracy is more sensitive to 
training parameters such as batch size and noise level than to the
structure of a neural network. 

%These include the regularization
%parameters, which affect generalization performance; the learning-rate
%schedule; and the network architecture, including the number and type of hidden
%layers and the number of units in each. In addition, for
%our algorithm, the lot size and the number of training steps are important hyperparameters.
%A larger lot size makes the
%estimate of the gradient less noisy but increases the privacy cost as
%well as the runtime. Running for more steps gets us closer to
%convergence but increases the privacy cost. \ktnote{possiby rephrase
%	and discuss a bit more, here or elsewhere.}
%
If we try several settings for the hyperparameters, we can trivially
add up the privacy costs of all the settings, possibly via the moments accountant. 
However, since we care only about the setting that gives us the most accurate model,
we can do better, such as applying a version of a result from Gupta et al.~\cite{GuptaLMRT10} \full{restated as Theorem~\ref{thm:hyperparameters} in the Appendix}{(see the full version of the paper for details~\protect\cite{DL-DP-arxiv})}.
%\manote{Omitting as superfluous the following, which did not seem easy to make nice: Chaudhuri et al.~\cite{} show
%	that if the learning algorithm satisfies a specific stability
%	condition, then one can get a better bound on the privacy cost. We do
%	not know if our algorithm satisfies this stability
%	condition}\ktnote{Check. Perhaps we can say that we know that it
%	doesn't.}

We can use insights from theory to reduce the number of hyperparameter settings that need to be tried.
While differentially private optimization of convex
objective functions is best achieved using batch sizes as small as 1,
non-convex learning, which is inherently less stable, benefits from
aggregation into larger batches. At the same time, Theorem~\ref{thm:main} suggests that making batches too large increases the privacy cost, and a 
reasonable tradeoff is to take the number of batches per epoch to be of the same order as the desired number of epochs.
The learning rate in non-private
training is commonly adjusted downwards carefully as the model
converges to a local optimum. In contrast, we never need to decrease the learning rate to a very small value, because differentially private training never reaches a regime where it would be justified. On the other hand, in our experiments, we do find that there is a small benefit to starting with a relatively large learning rate, then linearly decaying it to a smaller value in a few epochs, and keeping it constant afterwards.

%Private Amplification theorem from ``Differentially Private Combinatorial Optimization''
%Application to hyperparameter tuning.

%Perhaps discuss that we punt on this in the implementation. Or account for the additional cost.

\section{Implementation}\label{sec:impl} % Li
%% -*- TeX-master: "main" -*-

% Implementation details about dp_optimizer
%
% dp_optimizer accountant
% per-example gradient

We have implemented the differentially private SGD algorithms in
TensorFlow. The source code is available under an Apache 2.0 license from~\href{https://github.com/tensorflow/models}{github.com/tensorflow/models}.

For privacy protection, we need to ``sanitize''
the gradient before using it to update the parameters. In addition, we
need to keep track of the ``privacy spending'' based on how the
sanitization is done. Hence our implementation mainly consists of two
components: \texttt{sanitizer}, which preprocesses the gradient to
protect privacy, and \texttt{privacy\_accountant}, which keeps
track of the privacy spending over the course of training.

Figure~\ref{fig:snippet} contains the TensorFlow code snippet (in
Python) of \mbox{\texttt{DPSGD\_Optimizer}}, which minimizes a loss function
using a differentially private SGD, and \mbox{\texttt{DPTrain}}, which
iteratively invokes \texttt{DPSGD\_Optimizer} using a privacy accountant to bound the total privacy loss.

In many cases, the neural network model may benefit from the processing
of the input by projecting it on the principal directions (PCA) or by feeding it through
a convolutional layer. We implement differentially private PCA and apply pre-trained convolutional layers (learned on public data).

\begin{figure}[t]
\begin{Verbatim}[fontsize=\small]
class DPSGD_Optimizer():
  def __init__(self, accountant, sanitizer):
    self._accountant = accountant
    self._sanitizer = sanitizer

  def Minimize(self, loss, params,
               batch_size, noise_options):
    # Accumulate privacy spending before computing
    # and using the gradients.
    priv_accum_op =
        self._accountant.AccumulatePrivacySpending(
            batch_size, noise_options)
    with tf.control_dependencies(priv_accum_op):
      # Compute per example gradients
      px_grads = per_example_gradients(loss, params)
      # Sanitize gradients
      sanitized_grads = self._sanitizer.Sanitize(
          px_grads, noise_options)
      # Take a gradient descent step
      return apply_gradients(params, sanitized_grads)

def DPTrain(loss, params, batch_size, noise_options):
  accountant = PrivacyAccountant()
  sanitizer = Sanitizer()
  dp_opt = DPSGD_Optimizer(accountant, sanitizer)
  sgd_op = dp_opt.Minimize(
      loss, params, batch_size, noise_options)
  eps, delta = (0, 0)
  # Carry out the training as long as the privacy
  # is within the pre-set limit.
  while within_limit(eps, delta):
    sgd_op.run()
    eps, delta = accountant.GetSpentPrivacy()
\end{Verbatim}
\caption{Code snippet of \texttt{DPSGD\_Optimizer} and \texttt{DPTrain}.}\label{fig:snippet}
\end{figure}

\leader{Sanitizer.} In order to achieve privacy protection, the
 sanitizer needs to perform two operations: (1) limit the
 sensitivity of each individual example by clipping the norm of the
 gradient for each example; and (2) add noise to the gradient of a batch 
 before updating the network parameters.

In TensorFlow, the gradient computation is batched for performance reasons, yielding $\g_B = 1/|B| \sum_{x\in B} \nabla_\btheta\calL(\btheta, x)$ for a batch $B$ of training
examples. To limit the sensitivity of updates, we need to access each individual
$\nabla_\btheta \calL(\btheta, x)$. To this end, we
implemented \texttt{per\_example\_gradient} operator in TensorFlow, as
described by Goodfellow~\cite{Goodfellow15}. This operator can compute a batch of individual
$\nabla_\btheta \calL(\btheta,x)$.  With this implementation there is
only a modest slowdown in training, even for larger batch size.  Our current implementation
supports batched computation for the loss
function $\calL$, where each $x_i$ is singly connected to
$\calL$, allowing us to handle most hidden layers but not, for
example, convolutional layers.

Once we have the access to the per-example gradient, it is easy to
use TensorFlow operators to clip its norm and to add noise. 

\leader{Privacy accountant.} The main component in our implementation is \texttt{PrivacyAccountant}
which keeps track of privacy spending over the course of training.  As
discussed in Section~\ref{sec:approach}, we implemented the \logmgfa
accountant that additively accumulates the log of the moments of the
privacy loss at each step.  Dependent on the noise distribution, one
can compute $\alpha(\lambda)$ by either applying an asymptotic bound,
evaluating a closed-form expression, or applying numerical integration.  The first
option would recover the generic advanced composition theorem, and the
latter two give a more accurate accounting of the privacy loss.

For the Gaussian mechanism we use, $\alpha(\lambda)$ is defined
according to Eqs.~(\ref{eq:logmgf1}) and (\ref{eq:logmgf2}). In our
implementation, we carry out numerical integration to compute both
$E_1$ and $E_2$ in those equations. Also we compute $\alpha(\lambda)$
for a range of $\lambda$'s so we can compute the best possible
$(\eps,\delta)$ values using Theorem~\ref{thm:property}.\ref{thm:parttwo}. We find that for
the parameters of interest to us, it suffices to compute
$\alpha(\lambda)$ for $\lambda\leq 32$.

% \bmnote{ I suggest starting with something like: Simplifying the
%left-hand side of \e%qr{eq:logmgfdef}, we
%have \[ \E_{\outcome \sim \M(d)}[\exp(\lambda c(\outcome))
%= \E_{\outcome \sim \M(d)}\left[\left(\frac{\Pr[\M(d)
%= \outcome]}{\Pr[\M(d%') = \outcome]}\right)^\lambda\right].  \] }
%LZ: I have rewritten this part. PTAL. Thanks.  \bmnote{I think we
%should use $\lambda$ not $t$ here? But then there is a cl%ash of
%notation with the use of $\lambda$ at the end of this paragraph.}
%moment reduces to computing what we
%call \texttt{differential\_moment}, defined
%as \[\E_{\outcome \sim \M(d)}[\left(\Pr(\M(d')=x)/\Pr(\M(d)=x)-1\right)^t].\] \bmnote{I'm
%not sure it will be clear to the reader how to get from the equa%tion
%I added to here, I think we need to add a step or two. Also, we need
%to% get rid of the dependency on $d'$ here, since we can't compute
%anything tha%t depends on $d'$.}

At any point during training, one can query the privacy loss in the
more interpretable notion of $(\eps,\delta)$ privacy using
Theorem~\ref{thm:property}.\ref{thm:parttwo}. Rogers et al.~\cite{RRUV16-odometers} point out risks associated with adaptive choice of privacy parameters. We avoid their attacks and negative results by fixing the number of iterations and privacy parameters ahead of time. More general implementations of a privacy accountant must correctly distinguish between two modes of operation---as a privacy odometer or a privacy filter (see~{\protect\cite{RRUV16-odometers}} for more details).

\leader{Differentially private PCA.} Principal component analysis (PCA) is 
a useful method for capturing the main features of the input data.  We implement the differentially private PCA algorithm as described in~\cite{DworkTTZ14}. More
specifically, we take a random sample of the training examples, treat
them as vectors, and normalize each vector to unit $\ell_2$ norm to
form the matrix $A$, where each vector is a row in the matrix. We then
add Gaussian noise to the covariance matrix $A^T A$ and compute the
principal directions of the noisy covariance matrix.  Then for each
input example we apply the projection to these principal directions before
feeding it into the neural network. 

We incur a privacy cost due to running a PCA.  However, we find it useful for both
improving the model quality and for reducing the training time, as
suggested by our experiments on the MNIST data. See
Section~\ref{sec:impl} for details.

\leader{Convolutional layers.} Convolutional layers are useful for deep neural networks. However, an efficient per-example gradient computation for convolutional layers
remains a challenge within the TensorFlow framework, which motivates creating
a separate workflow. For example, some recent work argues that even random
convolutions often suffice~\cite{PSZC11, CoxPinto11, Saxe11,TuRVR16,DanielyFS16}.  

Alternatively, we explore the idea of learning convolutional layers on public
data, following Jarrett et al.~\cite{JKRL09}. Such convolutional layers can be based on GoogLeNet or AlexNet features~\cite{Szegedy-et-al-CVPR2015,alexnet} for image models or on pretrained word2vec or GloVe embeddings in language
models~\cite{word2vec,glove}.

\section{Experimental Results}\label{sec:results} % Li & Kunal
This section reports on our evaluation of the moments accountant, and results on two popular image datasets: MNIST and CIFAR-10.

\subsection{Applying the \Logmgfa Accountant}
%% -*- TeX-master: "main" -*-

As shown by Theorem~\ref{thm:main}, the \logmgfa accountant
provides a tighter bound on the privacy loss compared to the generic
strong composition theorem.  Here we compare them using some concrete
values.  The overall privacy loss $(\eps,\delta)$ can be computed from
the noise level $\sigma$, the sampling ratio of each lot $q=L/N$ (so
each epoch consists of $1/q$ batches), and the number of epochs $E$
(so the number of steps is $T=E/q$).  We fix the target $\delta
= 10^{-5}$, the value used for our MNIST and CIFAR experiments.

In our experiment, we set $q=0.01$, $\sigma=4$, and
$\delta=10^{-5}$, and compute the value of $\eps$ as a function
of the training epoch $E$.  Figure~\ref{fig:epoch-eps} shows two
curves corresponding to, respectively, using the strong composition
theorem and the \logmgfa accountant. We can see that we get a much
tighter estimation of the privacy loss by using the \logmgfa
accountant. For examples, when $E=100$, the values are $9.34$ and
$1.26$ respectively, and for $E=400$, the values are $24.22$ and
$2.55$ respectively. That is, using the \logmgfa bound, we achieve
$(2.55,10^{-5})$-differential privacy, whereas previous
techniques only obtain the significantly worse
guarantee of $(24.22, 10^{-5})$.

\begin{figure}[h]
{\centering \includegraphics[width=\columnwidth]{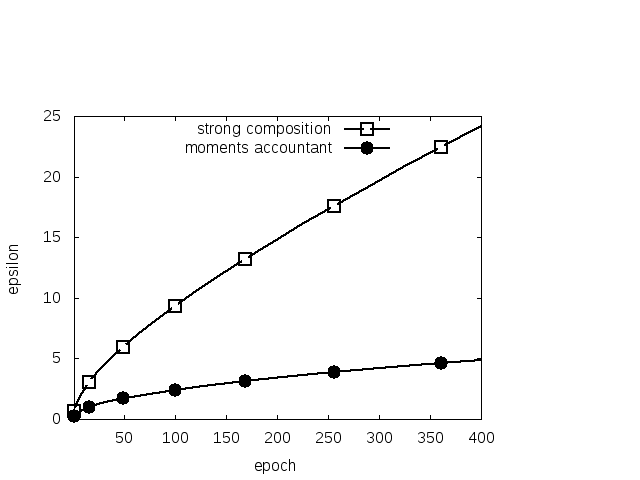}}
\caption{The $\eps$ value as a function of epoch $E$ for
 $q=0.01$, $\sigma=4$, $\delta=10^{-5}$, using the strong
 composition theorem and the \logmgfa accountant
 respectively.}\label{fig:epoch-eps}
\end{figure}

%In the second experiment, we fix the target privacy loss at $\eps=2.0$
%and $\delta=\texttt{1e-5}$ and the sampling ratio $q=0.01$. We
%estimate the noise we need to add as the function of the number of
%epoches we would like to run the training. As later shown in our
%experiments, the noise level affects the accuracy significantly. Hence
%if we set the number of epoches we would like to run, using strong
%composition theorem would require us to add a much larger noise which
%would in turn to cause big drop in accuracy.
%
%and $\delta=\texttt{1e-5}$ and the sampling ratio at $q=0.01$. We then
%calculate the noise we need to add as a function of the number of training
%epochs; Figure~\ref{fig:epoch-sigma} shows that the \logmgfa accountant allows us to add significantly less noise for any given number of training epochs. As our experiments below demonstrate, this reduction in the noise level translates to significant improvements in accuracy.\bmnote{It would be nice to quantify this here (and perhaps in the abstract/intro as well)}\lznote{This probably requires some explanation. We think Figure 1 already made the point. Also in the interest of space, we now take this figure out}
%\begin{figure}
%\includegraphics[width=2.25in]{data/epoch-sigma.png}
%\caption{The minimum noise needed as a function of training epochs for
% achieving $(2,\texttt{1e-5})$-differential privacy for $q=0.01$ using
% strong composition theorem and using \logmgfa accountant
% respectively.}\label{fig:epoch-sigma}
%\end{figure}

\subsection{MNIST} % Li
%% -*- TeX-master: "main" -*-
% Move to here to force Figure 3 to appear before Figure 4.
\begin{figure*}[t]
\begin{tabular}{ccc}
\includegraphics[width=2.25in]{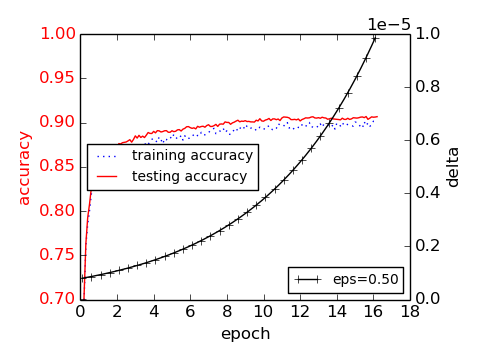} &
\includegraphics[width=2.25in]{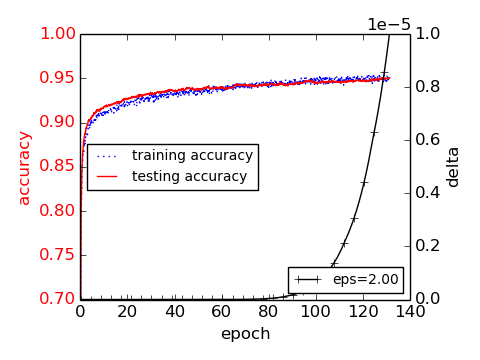} &
\includegraphics[width=2.25in]{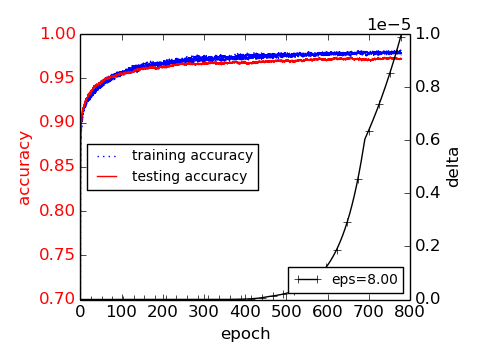} \\
(1) Large noise & (2) Medium noise & (3) Small noise\\
\end{tabular}
\caption{Results on the accuracy for different noise levels
  on the MNIST dataset. In all the experiments, the network
  uses $60$ dimension PCA projection, $1{,}000$ hidden units, and is
  trained using lot size $600$ and clipping threshold $4$. The noise
  levels $(\sigma, \sigma_p)$ for training the neural network and for
  PCA projection are set at ($8$, $16$), ($4$, $7$), and ($2$, $4$),
  respectively, for the three experiments.}\label{fig:bestmnist}
%\bmnote{I think the black lines should be labeled as $\delta$, with the fixed values of $\epsilon$ in the plot title? (though I think fixing delta and varying eps would be better yet). If we stick with $\delta$, a log-scale might be preferable (?). At least, if we care about $\delta < 1/N$, the $\delta$ scale doesn't show the regime of interest.} 
%LZ: per our offline discussion, will leave it as is for now.
\end{figure*}

We conduct experiments on the standard MNIST dataset for handwritten digit recognition consisting of  $60{,}000$ training examples and $10{,}000$ testing examples~\cite{mnist}. Each example is a
$28\times 28$ size gray-level image. We use a simple feed-forward neural
network with ReLU units and softmax of 10 classes (corresponding to
the 10 digits) with cross-entropy loss and an optional PCA input layer.

\paragraph{Baseline model}

Our baseline model uses a $60$-dimensional PCA projection layer and a single
hidden layer with $1{,}000$ hidden units. Using the lot size of $600$, we can
reach accuracy of $98.30\%$ in about $100$ epochs. This result is
consistent with what can be achieved with a vanilla neural
network~\cite{mnist}. 
%We note that it is quite easy to reach the
% accuracy of $97.5\%$ with about $200$ hidden units in $10$ epochs.

\paragraph{Differentially private model}

For the differentially private version, we experiment with the same
architecture with a $60$-dimensional PCA projection layer, a single
$1{,}000$-unit ReLU hidden layer, and a lot size of $600$.  To limit
sensitivity, we clip the gradient norm of each layer at $4$. We report
results for three choices of the noise scale, which we call small
($\sigma=2, \sigma_p=4$), medium ($\sigma=4, \sigma_p=7$), and large
($\sigma=8, \sigma_p=16$). Here $\sigma$ represents the noise level for
training the neural network, and $\sigma_p$ the noise level for PCA
projection. The learning rate is set at $0.1$ initially and linearly
decreased to $0.052$ over $10$ epochs and then fixed to $0.052$ thereafter.  We have also
experimented with multi-hidden-layer networks. For MNIST, we found
that one hidden layer combined with PCA works better than a two-layer
network.

Figure~\ref{fig:bestmnist}
shows the results for different noise levels. In each plot, we show
the evolution of the training and testing accuracy as a function of
the number of epochs as well as the corresponding $\delta$ value,
keeping $\eps$ fixed.  We achieve $90\%$, $95\%$, and $97\%$ test set
accuracy for $(0.5,10^{-5})$, $(2,10^{-5})$, and
$(8,10^{-5})$-differential privacy respectively.

One attractive consequence of applying differentially private SGD is the small difference between the model's accuracy on the training and the test sets, which is consistent with the theoretical argument that differentially private training generalizes well~\cite{BassilyNSSSU15}. In contrast, the gap between training and testing accuracy in non-private training, i.e., evidence of overfitting, increases with the number of epochs.

By using the \logmgfa accountant, we can obtain a $\delta$ value for any
given $\eps$. We record the accuracy for different $(\eps,\delta)$
pairs in Figure~\ref{fig:allmnist}. In the figure, each curve
corresponds to the best accuracy achieved for a fixed $\delta$, as it
varies between $10^{-5}$ and $10^{-2}$. For example, we
can achieve $90\%$ accuracy for $\eps=0.25$ and $\delta=0.01$. As can
be observed from the figure, for a fixed $\delta$, varying the value
of $\eps$ can have large impact on accuracy, but for any fixed
$\eps$, there is less difference with different $\delta$ values.

\begin{figure}[h]
\begin{tabular}{c}
\includegraphics[width=3.in]{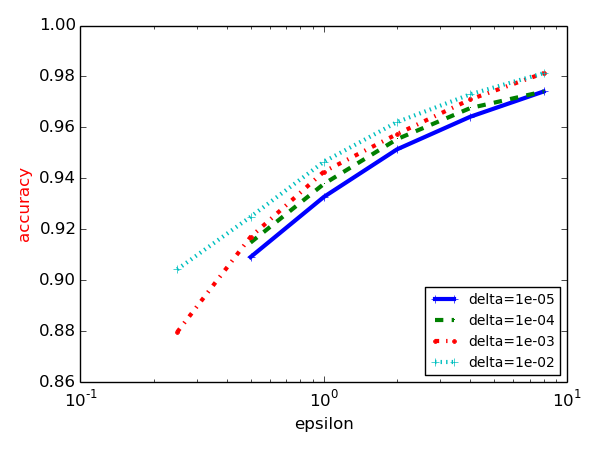}
\end{tabular}
\caption{Accuracy of various $(\eps, \delta)$ privacy values on the MNIST dataset. Each curve corresponds to a different $\delta$ value.}\label{fig:allmnist}
\end{figure}

\paragraph{Effect of the parameters}

Classification accuracy is determined by multiple factors that must be
carefully tuned for optimal performance. These factors include the
topology of the network, the number of PCA dimensions and the number
of hidden units, as well as parameters of the training procedure such
as the lot size and the learning rate.  Some parameters are specific
to privacy, such as the gradient norm clipping bound and the noise
level.

\begin{figure*}[t]
\begin{tabular}{ccc}
\includegraphics[width=2.25in]{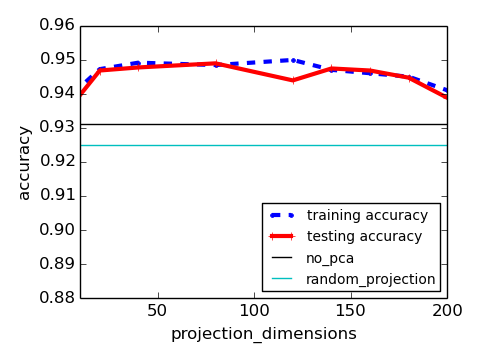} &
\includegraphics[width=2.25in]{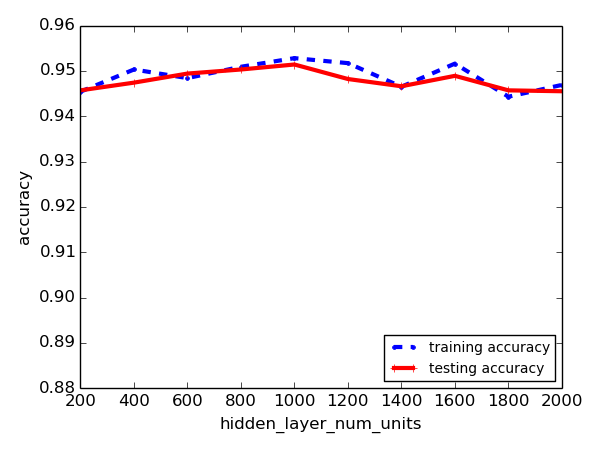} &
\includegraphics[width=2.25in]{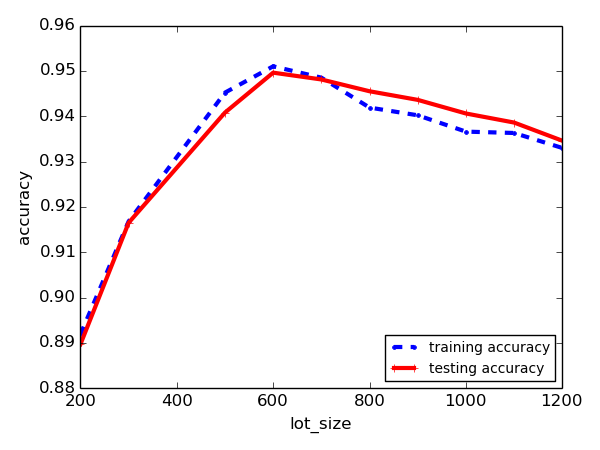} \\
(1) variable projection dimensions & (2) variable hidden units & (3) variable lot size\\
\includegraphics[width=2.25in]{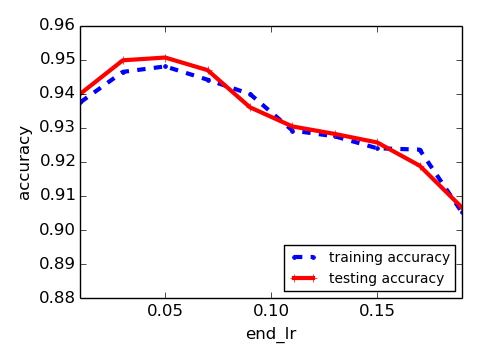} &
\includegraphics[width=2.25in]{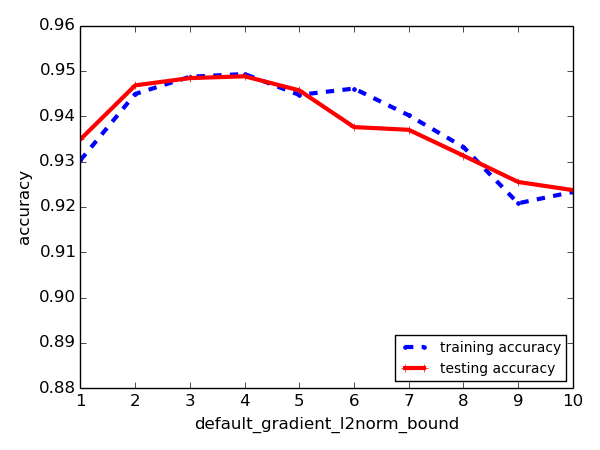} &
\includegraphics[width=2.25in]{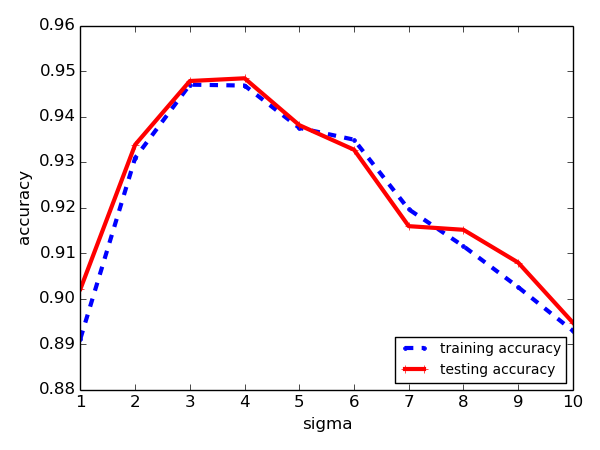}\\
(4) variable learning rate &  (5) variable gradient clipping norm & (6) variable noise level\\
\end{tabular}
\caption{MNIST accuracy when one parameter varies, and the others are
  fixed at reference values.}\label{fig:var}
\end{figure*}

To demonstrate the effects of these parameters, we manipulate them
individually, keeping the rest constant. We set the reference values as follows: 60 PCA dimensions, 1{,}000 hidden units, 600 lot size, gradient norm bound of 4, initial learning rate of 0.1 decreasing to a final learning rate of 0.052 in 10 epochs, and noise $\sigma$ equal to 4 and 7 respectively for training the neural network parameters and for the PCA projection. For each combination of values, we train until the
point at which $(2,10^{-5})$-differential privacy would be violated
(so, for example, a larger $\sigma$ allows more epochs of training). The
results are presented in Figure~\ref{fig:var}.

%As explained in Section~\ref{sec:hyper},
%differentially private training is less sensitive to the adjustment of
%the learning rate. In our experiment we use a fixed learning rate so
%it is not included in the analysis.

\leader{PCA projection.} In our experiments, the accuracy is fairly
stable as a function of the PCA dimension, with the best results
achieved for $60$. (Not doing PCA reduces accuracy by about $2\%$.)
Although in principle the PCA projection layer can be replaced by an
additional hidden layer, we achieve better accuracy by training the
PCA layer separately. By reducing the input size from $784$ to $60$,
PCA leads to an almost $10\times$ reduction in training time. The
result is fairly stable over a large range of the noise levels
for the PCA projection and consistently better than the accuracy using random
projection, which is at about $92.5\%$ and shown as a horizontal line in the plot.

\leader{Number of hidden units.} Including more hidden units makes it
easier to fit the training set. For non-private training, it is often
preferable to use more units, as long as we employ techniques to avoid
overfitting. However, for differentially private training, it is not a
priori clear if more hidden units improve accuracy, as more hidden
units increase the sensitivity of the gradient, which leads to more
noise added at each update.
 
Somewhat counterintuitively, increasing the number of hidden units
does not decrease accuracy of the trained model. One possible
explanation that calls for further analysis is that larger networks
are more tolerant to noise.  This property is quite encouraging as it
is common in practice to use very large networks. 
%\bmnote{(side note,
%  no need to update the text) These feels related to the fact that with the
%  Gaussian mechanism, if you are trying to estimate a random variable
%  $X$, you get the same estimator if your input is $[X]$, or the
%  vector $[X, X, \dots, X]$ --- the additional noise exactly cancels out.}
%LZ: Good point. It would be good to have some theoretical understanding
%of this phenomenon.

\leader{Lot size.}  According to Theorem~\ref{thm:main}, we can run
$N/L$ epochs while staying within a constant privacy budget. Choosing
the lot size must balance two conflicting objectives. On the one hand,
smaller lots allow running more epochs, i.e., passes over data,
improving accuracy. On the other hand, for a larger lot, the added
noise has a smaller relative effect.

Our experiments show that the lot size has a relatively large impact on
accuracy. Empirically, the best lot size is roughly $\sqrt{N}$ where
$N$ is the number of training examples.
%\bmnote{I'm not actually sure what theory this is. By
% epoch do we mean one pass over the training data? I would expect
% larger batches we need more epochs, to keep the total number of
% update steps constant.}\lznote{Actually by the privacy amplication by
% sampling, we can ran $N/B$ epochs without incurring additional
% privacy loss as running one epoch. This is from privacy view. I am
% not sure about the theory about the influence of batch size on the
% convergence. It would be good to say something about it.} 

\leader{Learning rate.} Accuracy is stable for a learning rate in the
range of $[0.01, 0.07]$ and peaks at 0.05, as shown in
Figure~\ref{fig:var}(4). However, accuracy decreases significantly if
the learning rate is too large. Some additional experiments suggest
that, even for large learning rates, we can reach similar levels of
accuracy by reducing the noise level and, accordingly, by training
less in order to avoid exhausting the privacy budget.

\leader{Clipping bound.} Limiting the gradient norm has two opposing
effects: clipping destroys the unbiasedness of the gradient estimate,
and if the clipping parameter is too small, the average clipped
gradient may point in a very different direction from the true
gradient. On the other hand, increasing the norm bound $C$ forces us
to add more noise to the gradients (and hence the parameters), since
we add noise based on $\sigma C$. In practice, a good way to choose a
value for $C$ is by taking the median of the norms of the unclipped
gradients over the course of training.

\leader{Noise level.} By adding more noise, the per-step
privacy loss is proportionally smaller, so we can run more epochs within a given cumulative privacy budget. In Figure~\ref{fig:var}(5), the $x$-axis is
the noise level $\sigma$.
%\bmnote{Can we just say $\sigma$, or do we mean something different?}
%LZ: fixed.
The choice of this value has a large impact on accuracy.

\bigskip

From the experiments, we observe the following.

\begin{enumerate}

\item The PCA projection improves both model accuracy and training
  performance. Accuracy is quite stable over a large range of choices for
  the projection dimensions and the noise level used in the PCA stage.

 \item The accuracy is fairly stable over the network size. When we
 can only run smaller number of epochs, it is more beneficial to use
 a larger network.

 \item The training parameters, especially the lot size and the
 noise scale $\sigma$, have a large impact on the model
 accuracy. They both determine the ``noise-to-signal'' ratio of the
 sanitized gradients as well as the number of epochs we are able to
 go through the data before reaching the privacy limit.
\end{enumerate}

Our framework allows for adaptive control of the
training parameters, such as the lot size, the gradient norm bound $C$, and
noise level $\sigma$. Our initial experiments with decreasing noise as
training progresses did not show a significant improvement, but it
is interesting to consider more sophisticated schemes for adaptively
choosing these parameters.

% Description of the dataset 
% Guidance on choosing parameters, topology

\subsection{CIFAR} % Kunal
%% -*- TeX-master: "main" -*-
\begin{figure*}[!]
	\begin{tabular}{ccc}
		\includegraphics[width=2.25in]{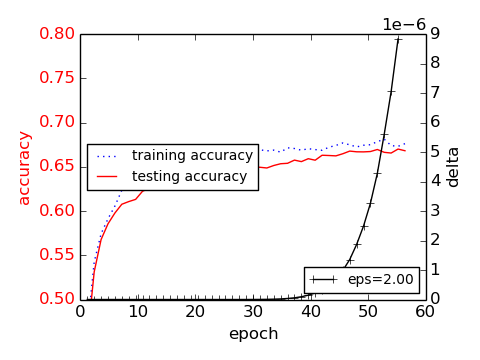}&
		\includegraphics[width=2.25in]{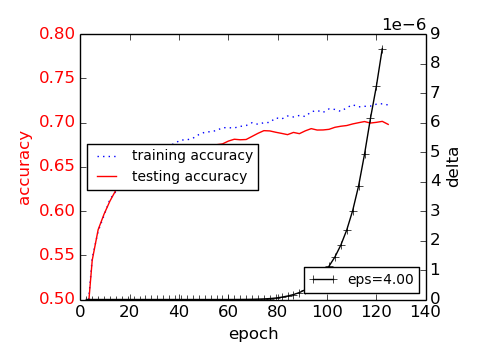}&
		\includegraphics[width=2.25in]{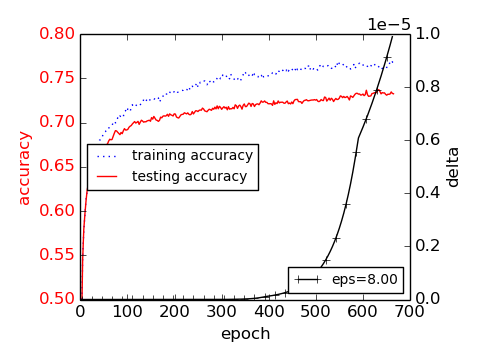}\\
		(1) $\eps=2$ & (2) $\eps=4$ & (3) $\eps=8$\\
	\end{tabular}
	\caption{Results on accuracy for different noise levels on CIFAR-10. With $\delta$ set to $10^{-5}$, we achieve accuracy $67\%$, $70\%$, and $73\%$, with $\eps$ being $2$, $4$, and $8$, respectively. The first graph uses a lot size of 2{,}000, (2) and (3) use a lot size of 4{,}000. In all cases, $\sigma$ is set to $6$, and clipping is set to~$3$.
	}\label{fig:cifar}
\end{figure*}
We also conduct experiments on the CIFAR-10 dataset, which consists of color images classified into 10 classes such as ships, cats, and dogs, and partitioned into $50{,}000$ training examples and $10{,}000$ test examples~\cite{cifar_webpage}. Each example is a $32\times32$ image with three channels (RGB). For this learning task, nearly all successful networks use convolutional layers. The CIFAR-100 dataset has similar parameters, except that images are classified into 100 classes; the examples and the image classes are different from those of CIFAR-10.

We use the network architecture from the TensorFlow convolutional neural networks tutorial~\cite{tensorflow_cnn}. Each $32\times 32$ image is first cropped to a $24\times 24$ one by taking the center patch.  The network architecture consists of two convolutional layers followed by two fully connected layers. The convolutional layers use $5\times 5$ convolutions with stride 1, followed by a ReLU and $2\times 2$ max pools, with 64 channels each. Thus the first convolution outputs a $12 \times 12 \times 64$ tensor for each image, and the second outputs a $6\times 6 \times 64$ tensor. The latter is flattened to a vector that gets fed into a fully connected layer with $384$ units, and another one of the same size. %This network architecture can get to $86\%$ accuracy.

This architecture, non-privately, can get to about $86\%$ accuracy in $500$
epochs. Its simplicity makes it an appealing choice for our work. We
should note however that by using deeper networks with different
non-linearities and other advanced techniques, one can obtain
significantly better accuracy, with the state-of-the-art being
about $96.5\%$~\cite{Graham14a}.
 
As is standard for such image datasets, we use {\em data augmentation}
during training. For each training image, we generate a new distorted
image by randomly picking a $24\times 24$ patch from the image,
randomly flipping the image along the left-right direction, and
randomly distorting the brightness and the contrast of the image. In
each epoch, these distortions are done independently. We refer the
reader to the TensorFlow tutorial~\cite{tensorflow_cnn} for additional details. 

As the convolutional layers have shared parameters, computing
per-example gradients has a larger computational overhead. Previous
work has shown that convolutional layers are often transferable:
parameters learned from one data\-set can be used on another one without
retraining~\cite{JKRL09}. We treat the CIFAR-100 dataset as a public dataset and use
it to train a network with the same architecture. We use the
convolutions learned from training this dataset. Retraining only the
fully connected layers with this architecture for about 250 epochs
with a batch size of 120 gives us approximately $80\%$ accuracy, which
is our non-private baseline.

\paragraph{Differentially private version}

For the differentially private version, we use the same
architecture. As discussed above, we use pre-trained convolutional layers. The
fully connected layers are initialized from the pre-trained network as
well.  We train the softmax layer, and either the top or both fully
connected layers. Based on looking at gradient norms, the softmax layer
gradients are roughly twice as large as the other two layers, and we
keep this ratio when we try clipping at a few different values between
3 and 10. The lot size is an additional knob that we tune: we
tried $600$, $2{,}000$, and $4{,}000$. With these settings, the per-epoch training time increases from approximately 40 seconds to 180 seconds.

In Figure~\ref{fig:cifar}, we show the evolution of the accuracy and
the privacy cost, as a function of the number of epochs, for a few
different parameter settings.

The various parameters influence the accuracy one gets, in ways not
too different from that in the MNIST experiments. A lot size of 600
leads to poor results on this dataset and we need to increase it to
2{,}000 or more for results reported in Figure~\ref{fig:cifar}.

Compared to the MNIST dataset, where the difference in accuracy between a non-private baseline and a private model is about 1.3\%, the corresponding drop in accuracy in our CIFAR-10 experiment is much larger (about 7\%). We leave closing this gap as an interesting test for future research in differentially private machine learning.

% Convolutions ?
%   - full-featured ??
%   - w/o training 
%   - based on public dataset ?

\section{Related Work}\label{sec:related} % Ilya
The problem of privacy-preserving data mining, or machine learning, has been a focus of active work in several research communities since the late 90s~\cite{AgrawalS00,LindellPinkas00}. The existing literature can be broadly classified along several axes: the class of models, the learning algorithm, and the privacy guarantees.

\leader{Privacy guarantees.} Early works on privacy-preserving learning were done in the framework of secure function evaluation (SFE) and secure multi-party computations (MPC), where the input is split between two or more parties, and the focus is on minimizing information leaked during the joint computation of some agreed-to functionality. In contrast, we assume that data is held centrally, and we are concerned with leakage from the functionality's output (i.e., the model).

Another approach, $k$-anonymity and closely related notions~\cite{sweeney2002-k-anonymity}, seeks to offer a degree of protection to underlying data by generalizing and suppressing certain identifying attributes. The approach has strong theoretical and empirical limitations~\cite {Aggarwal-curse,Brickell-Shmatikov} that make it all but inapplicable to de-anonymization of high-dimensional, diverse input datasets. Rather than pursue input sanitization, we keep the underlying raw records intact and perturb derived data instead.

The theory of differential privacy, which provides the analytical framework for our work, has been applied to a large collection of machine learning tasks that differed from ours either in the training mechanism or in the target model.

The \logmgfa accountant is closely related to the notion of R\'enyi differential privacy~\cite{Mironov16Renyi}, which proposes (scaled) $\alpha(\lambda)$ as a means of quantifying privacy guarantees. In a concurrent and independent work Bun and Steinke~\cite{BunS16} introduce a relaxation of differential privacy (generalizing the work of Dwork and Rothblum~\cite{DworkRoth14}) defined via a linear upper bound on $\alpha(\lambda)$. Taken together, these works demonstrate that the \logmgfa accountant is a useful technique for theoretical and empirical analyses of complex privacy-preserving algorithms.

\leader{Learning algorithm.} A common target for learning with privacy is a class of convex optimization problems amenable to a wide variety of techniques~\cite{DworkLei09,chaudhuri2011,KST12}. In concurrent work, Wu et al.~achieve 83\% accuracy on MNIST via convex empirical risk minimization~\cite{WCJN16-RDBMS}. Training multi-layer neural networks is non-convex, and typically solved by an application of SGD, whose theoretical guarantees are poorly understood.

For the CIFAR neural network we incorporate differentially private training of the PCA projection matrix~\cite{DworkTTZ14}, which is used to reduce dimensionality of inputs.

\leader{Model class.} The first end-to-end differentially private system was evaluated on the Netflix Prize dataset~\cite{McSherryMironov09}, a version of a collaborative filtering problem. Although the problem shared many similarities with ours---high-dimensional inputs, non-convex objective function---the approach taken by McSherry and Mironov differed significantly. They identified the core of the learning task, effectively sufficient statistics, that can be computed in a differentially private manner via a Gaussian mechanism. In our approach no such sufficient statistics exist.

In a recent work Shokri and Shmatikov~\cite{ShokriShmatikov15} designed and evaluated a system for \emph{distributed} training of a deep neural network. Participants, who hold their data closely, communicate sanitized updates to a central authority. The sanitization relies on an additive-noise mechanism, based on a sensitivity estimate, which could be improved to a hard sensitivity guarantee. They compute privacy loss per parameter (not for an entire model). By our preferred measure, the total privacy loss per participant on the MNIST dataset exceeds several thousand.

A different, recent approach towards differentially private deep learning is explored by Phan et al.~\cite{PhanWWD16}. This work focuses on learning autoencoders. Privacy is based on perturbing the objective functions of these autoencoders.

\section{Conclusions}\label{sec:conclusions}
We demonstrate the training of deep neural
networks with differential privacy, incurring a modest total privacy
loss, computed over entire models with many parameters.  In our
experiments for MNIST, we achieve 97\% training accuracy and for CIFAR-10 we achieve 73\% accuracy, both with
$(8,10^{-5})$-differential privacy.  Our algorithms are based on a differentially private
version of stochastic gradient descent; they run on the TensorFlow
software library for machine learning. Since our approach applies directly to gradient computations, it can be adapted to many other classical and more recent first-order optimization methods, such as NAG~\cite{Nesterov}, Momentum~\cite{RHW-Momentum}, AdaGrad~\cite{DHS11-AdaGrad}, or SVRG~\cite{JognsonZhang13-SVRG}.

A new tool, which may be of independent interest, is a 
mechanism for tracking privacy loss, the moments accountant. It permits tight automated analysis of the privacy loss of complex composite mechanisms that are currently beyond the reach of advanced composition theorems.

A number of avenues for further work are attractive. In particular, we
would like to consider other classes of deep networks. Our experience with
MNIST and CIFAR-10 should be helpful, but we see many opportunities for new research, for example in applying our techniques to LSTMs used for language modeling tasks. In addition, we would like to obtain additional improvements in
accuracy. Many training datasets are much larger than those of MNIST
and CIFAR-10; accuracy should benefit from their size.

\section{Acknowledgments}\label{sec:acknowledgments}
We are grateful to {\'U}lfar Erlingsson and Dan Ramage for many useful discussions, and to Mark Bun and Thomas Steinke for sharing a draft of~\cite{BunS16}.
\vfill

\bibliographystyle{abbrv}
\bibliography{dp-dl}

\full{
\clearpage
\appendix
\newcommand{\e}{\mathbf{e}}
\newcommand{\zero}{\mathbf{0}}
\newcommand{\ud}{\,\mbox{d}}
\def\protectedrefA{\ref{thm:property}}
\section{Proof of Theorem~\protectedrefA} % otherwise the reference gets capitalized
Here we restate and prove Theorem~\ref{thm:property}.
\setcounter{thm}{1}\begin{thm}\label{thm:property_supp}
Let $\alpha_\M(\lambda)$ defined as \[\alpha_\M(\lambda) \eqdef \max_{\aux, d, d'} \alpha_\M(\lambda; \aux, d, d'),\]
where the maximum is taken over all auxiliary inputs and neighboring databases $d,d'$. Then

\begin{enumerate}
\item \textbf{[Composability]}
Suppose that a mechanism $\M$ consists of a sequence of adaptive mechanisms $\M_1, \ldots, \M_k$ where $\M_i\colon \prod_{j=1}^{i-1}\Range_j\times \Domain \to\Range_i$. Then, for any $\lambda$
\[\alpha_\M(\lambda) \leq \sum_{i=1}^k \alpha_{\M_i}(\lambda)\,.\]

\item \textbf{[Tail bound]}
For any $\eps>0$, the mechanism $\M$ is $(\eps, \delta)$-differentially private for
\[\delta=\min_{\lambda} \exp(\alpha_\M(\lambda) -\lambda \eps)\,.\]
\end{enumerate}
\end{thm}

\begin{proof}
\leader{Composition of moments.}
  For brevity, let $\M_{1:i}$ denote $(\M_1, \ldots, \M_i)$, and similarly let $\outcome_{1:i}$ denote $(\outcome_1, \ldots, \outcome_i)$. For neighboring databases $d, d' \in D^n$, and a sequence of outcomes $\outcome_1, \ldots, \outcome_k$ we write
\begin{align*}
c(&\outcome_{1:k}; \M_{1:k}, o_{1:(k-1)}, d, d')\\ &= \log\frac{\Pr[\M_{1:k}(d; o_{1:(k-1)}) = \outcome_{1:k}]}{\Pr[\M_{1:k}(d'; o_{1:(k-1)}) = \outcome_{1:k}]}\\
&=  \log \prod_{i=1}^k \frac{\Pr[\M_i(d) = \outcome_i \mid \M_{1:(i-1)}(d) = \outcome_{1:(i-1)}]}{\Pr[\M_i(d') = \outcome_i \mid \M_{1:(i-1)}(d') = \outcome_{1:(i-1)}]}\\
&= \sum_{i=1}^k \log \frac{\Pr[\M_i(d) = \outcome_i \mid \M_{1:(i-1)}(d) = \outcome_{1:(i-1)}]}{\Pr[\M_i(d') = \outcome_i \mid \M_{1:(i-1)}(d') = \outcome_{1:(i-1)}]}\\
&= \sum_{i=1}^k c(\outcome_i; \M_i, \outcome_{1:(i-1)},d, d').
  \end{align*}
Thus 
\begin{align*}
  \E&_{\outcome'_{1:k} \sim \M_{1:k}(d)}\big[\exp(\lambda c(\outcome'_{1:k} ; \M_{1:k}, d, d'))\bigm| \forall i<k\colon o'_i=o_i\big]\\
  &\;\;\;=   \E_{\outcome'_{1:k} \sim \M_{1:k}(d)}\left[\exp\left(\lambda \sum_{i=1}^k c(\outcome'_i; \M_i, \outcome_{1:(i-1)},d, d')\right)\right]\\
    &\;\;\;=  \E_{\outcome'_{1:k} \sim \M_{1:k}(d)}\left[\prod_{i=1}^k\exp\big(\lambda c(\outcome'_i; \M_i,\outcome_{1:(i-1)}, d, d')\big)\right]\tag{by independence of noise}\\
      &\;\;\;= \prod_{i=1}^k \E_{\outcome'_{i} \sim \M_{i}(d)}\left[\exp(\lambda c(\outcome'_i; \M_i, \outcome_{1:(i-1)}, d, d')) \right]\\
        &\;\;\;= \prod_{i=1}^k \exp\big(\alpha_{\M_i}(\lambda;\outcome_{1:(i-1)}, d, d')\big)\\
        &\;\;\;= \exp\left(\sum_{i=1}^k \alpha_i(\lambda;\outcome_{1:(i-1)}, d, d')\right).
  \end{align*}
  The claim follows.

\leader{Tail bound by moments.}
  The proof is based on the standard Markov's inequality argument used in proofs of measure concentration. We have
  \begin{align*}
    \Pr_{\outcome \sim \M(d)}&[c(\outcome) \geq \eps]\\ &= \Pr_{\outcome \sim \M(d)}[\exp(\lambda c(\outcome)) \geq \exp(\lambda\eps))]\\
    &\leq \frac{\E_{\outcome \sim \M(d)}[\exp(\lambda c(\outcome))]}{\exp(\lambda\eps)}\\
      &\leq \exp(\alpha - \lambda\eps).
    \end{align*}
Let $B = \{\outcome\colon c(\outcome) \geq \eps\}$. Then for any $S$,
\begin{align*}\Pr[&M(d) \in S] \\&= \Pr[M(d) \in S \cap B^c] + \Pr[M(d)
    \in S \cap B]\\
&\leq \exp(\eps)\Pr[M(d') \in S \cap B^c] + \Pr[M(d) \in B]\\
&\leq \exp(\eps)\Pr[M(d') \in S] + \exp(\alpha  - \lambda \eps).
\end{align*}
  The second part follows by an easy calculation.
  \end{proof}
The proof demonstrates a tail bound on the
privacy loss, making it stronger than differential privacy for a fixed value
of $\eps,\delta$.

\mycomment{If one wants to get $(\eps,\delta)$-differential privacy for the
composition of $k$ Gaussian mechanisms, the standard argument of going
through the strong composition theorem would require
$\sigma^2 \approxeq
k\log \frac{2}{\delta} \log\frac{3k}{\delta}/\eps^2$. In contrast,
using lemma~\ref{lem:gaussian_logmgf}, along with
Theorem~\ref{thm:property_supp}
allows us to take $\sigma^2 \approxeq 4k \log \frac{1}{\delta}
/ \eps^2$. For $\frac{3k}{\delta} > \exp(4)$, the latter bound is
better.
}

\def\protectedrefB{\ref{lem:sampled_gaussian_mgf}}
\section{Proof of Lemma~\protectedrefB} % otherwise the reference gets capitalized

The proof of the main theorem relies on the following moments bound on
Gaussian mechanism with random sampling.
\begin{lem}
\label{lem:sampled_gaussian_mgf}
Suppose that $f\colon D \rightarrow \R^p$ with $\|f(\cdot)\|_2 \leq 1$. Let $\sigma \geq 1$ and let $J$ be a sample from $[n]$ where each $i \in [n]$ is chosen independently with probability $q < \frac{1}{16\sigma}$. Then for any positive integer $\lambda \leq \sigma^2\ln \frac{1}{q\sigma}$, the mechanism $\M(d) = \sum_{i \in J}f(d_i) + \calN(0,\sigma^2 \Id)$ satisfies  %$(\alpha, \lambda)$-logmgf bounded for any positive integer $\lambda \leq \sigma/4q$, where
\begin{align*}
\alpha_{\M}(\lambda) &\leq \frac{q^2\lambda(\lambda+1)}{(1-q)\sigma^2} + O(q^3\lambda^3/\sigma^3).
  \end{align*}
\end{lem}
\begin{proof}
Fix $d'$ and let $d = d' \cup \{d_n\}$. Without loss
of generality, $f(d_n) = \e_1$ and $\sum_{i \in J \setminus [n]}f(d_i)
= \zero$. Thus $\M(d)$ and $\M(d')$ are distributed identically except for the first coordinate
and hence we have a one-dimensional problem. Let $\mu_0$ denote the
pdf of $\calN(0, \sigma^2)$ and let $\mu_1$ denote the pdf of
$\calN(1, \sigma^2)$. Thus: 
\begin{align*}
\M(d') &\sim \mu_0,\\
\M(d) &\sim \mu \eqdef{\ss} (1-q)\mu_0 + q\mu_1.%\left\{ \begin{array}{ll}\calN(0, \sigma^2) & \mbox{w.p. } (1-q)\\N(1, \sigma^2) & \mbox{w.p. } q\\\end{array} \right.
\end{align*}
We want to show that
\begin{align*}
\E_{z \sim \mu} [ (\mu(z) / \mu_0(z))^\lambda] &\leq \alpha,\\
\mbox{and  }{}\E_{z \sim \mu_0} [ (\mu_0(z) / \mu(z))^\lambda] &\leq \alpha,
\end{align*}
for some explicit $\alpha$ to be determined later.
 
%Note that for any probability distributions $\nu_0,\nu_1$ and $\lambda\geq 1$,  the $(\lambda+1)$-th order R\'enyi divergence 
%\[\frac{1}{\lambda}\E_{z\sim \nu_0} [(\nu_0(z)/\nu_1(z))^{\lambda}]\,,\]
%is monotonically increasing in $\lambda$~\cite{vEH14-Renyi}.  It therefore suffices to show the bounds for integral $\lambda$.

We will use the same method to prove both bounds. Assume we have two distributions $\nu_0$ and $\nu_1$, and we wish to bound
\[\E_{z \sim \nu_0} [ (\nu_0(z) / \nu_1(z))^\lambda]=\E_{z \sim \nu_1} [ (\nu_0(z) / \nu_1(z))^{\lambda+1}]\,.\]

Using binomial expansion, we have
\begin{align}
{}&\E_{z \sim \nu_1} [ (\nu_0(z) / \nu_1(z))^{\lambda+1}]\nonumber\\
={}&\E_{z \sim \nu_1} [ (1 + (\nu_0(z)-\nu_1(z)) / \nu_1(z))^{\lambda+1}]\nonumber\\
={}&\E_{z \sim \nu_1} [ (1 + (\nu_0(z)-\nu_1(z)) / \nu_1(z))^{\lambda+1}]\nonumber\\
={}&\sum_{t=0}^{\lambda+1} {\lambda+1\choose t} \E_{z \sim \nu_1} [((\nu_0(z)-\nu_1(z)) / \nu_1(z))^t]\,.\label{eqn:binomial}
\end{align}

The first term in (\ref{eqn:binomial}) is $1$, and the second term is
\begin{align*}
\E_{z \sim \nu_1}\left[\frac{\nu_0(z) - \nu_1(z)}{\nu_1(z)}\right] &= \int_{-\infty}^\infty \nu_1(z)\frac{\nu_0(z) - \nu_1(z)}{\nu_1(z)} \ud z\\
&= \int_{-\infty}^\infty \nu_0(z) \ud z - \int_{-\infty}^\infty \nu_1(z) \ud z \\
&= 1 - 1 = 0.
\end{align*}

To prove the lemma it suffices to show show that for both $\nu_0 = \mu, \nu_1=\mu_0$ and $\nu_0=\mu_0,\nu_1=\mu$, the third term is bounded by $q^2\lambda(\lambda+1)/(1-q)\sigma^2$
%\[{\lambda+1 \choose 2}q^2 (\exp(1/(\sigma^2))-1)\,,\]
and that this bound dominates the sum of the remaining terms. We will prove the more difficult second case
($\nu_0=\mu_0,\nu_1=\mu$); the proof of the other case is similar.

To upper bound the third term in (\ref{eqn:binomial}), we note that $\mu(z) \geq
(1-q)\mu_0(z)$, and write
\begin{align*}
\E_{z \sim \mu}&\left[\left(\frac{\mu_0(z) - \mu(z)}{\mu(z)}\right)^2\right]\\
 &= q^2\E_{z \sim \mu}\left[\left(\frac{\mu_0(z) - \mu_1(z)}{\mu(z)}\right)^2\right]\\
 &= q^{2}\int_{-\infty}^\infty \frac{(\mu_0(z) - \mu_1(z))^{2}}{\mu(z)} \ud z\\
 &\leq \frac{q^{2}}{1-q} \int_{-\infty}^\infty \frac{(\mu_0(z) - \mu_1(z))^{2}}{\mu_{0}(z)} \ud z\\
 &= \frac{q^{2}}{1-q} \E_{z \sim \mu_0}\left[\left(\frac{\mu_0(z) - \mu_1(z)}{\mu_0(z)}\right)^2\right].
\end{align*}
An easy fact is that for any $a \in \R$, $\E_{z \sim \mu_0} \exp(2az/2\sigma^2)
= \exp(a^2/2\sigma^2)$. Thus,
\begin{align*}
\E_{z \sim \mu_0}&\left[\left(\frac{\mu_0(z)
 - \mu_1(z)}{\mu_0(z)}\right)^2\right]\\ &= \E_{z \sim \mu_0} \left[\left(1
   - \exp(\frac{2z-1}{2\sigma^2})\right)^2\right]\\
&= 1 - 2\E_{z \sim \mu_0}\left[\exp(\frac{2z-1}{2\sigma^2})\right]\\
&\hphantom{=1\;}  + \E_{z \sim \mu_0}\left[\exp(\frac{4z-2}{2\sigma^2})\right]\\
&= 1 - 2\exp\left(\frac{1}{2\sigma^{2}}\right) \cdot \exp\left(\frac{-1}{2\sigma^{2}}\right)\\
&\hphantom{=1\;} + \exp\left(\frac{4}{2\sigma^{2}}\right) \cdot \exp\left(\frac{-2}{2\sigma^{2}}\right)\\
&= \exp(1/\sigma^2) - 1.\\% \leq \frac{2}{\sigma^{2}}\\
\end{align*}

Thus the third term in the binomial expansion (\ref{eqn:binomial})
\begin{align*}
{{1+\lambda} \choose 2} \E_{z \in \mu} \left[\left(\frac{\mu_0(z) - \mu(z)}{\mu(z)}\right)^2\right] &\leq  \frac{\lambda(\lambda+1) q^{2}}{(1-q)\sigma^{2}}.
\end{align*}
% &\leq (\frac{q}{1-q})^2\left((1-q)\E_{z \sim \mu_0}[(\frac{\mu_0(z)
% - \mu_1(z)}{\mu_0(z)})^2]\right.\\
%&\,\,\,\, \,\,\left.+ q\E_{z \sim \mu_1}[(\frac{\mu_0(z) - \mu_1(z)}{\mu_0(z)})^2]\right).\\
%\end{align*}
%An easy fact is that $\E_{z \sim \mu_0} \exp(2az/2\sigma^2)
%= \exp(a^2/2\sigma^2)$. Thus,
%\begin{align*}
%\E_{z \sim \mu_0}&[(\frac{\mu_0(z)
% - \mu_1(z)}{\mu_0(z)})^2]\\ &= \E_{z \sim \mu_0} [(1
%   - \exp(\frac{2z-1}{2\sigma^2}))^2]\\
%&= 1 - 2\E_{z \sim \mu_0}[\exp(\frac{2z-1}{2\sigma^2})]
%    + \E_{z \sim \mu_0}[\exp(\frac{4z-2}{2\sigma^2})]\\
%&= \exp(2/2\sigma^2) - 1.\\
%\end{align*}
%
%A similar calculation shows that the second term is also
%$\exp(6/2\sigma^2) - 2\exp(2/2\sigma^2) + 1$. Thus
%\begin{align*}
%\E_{z \sim \mu}&[(\frac{\mu_0(z) - \mu(z)}{\mu(z)})^2]\\
%&\leq (\frac{q}{1-q})^2 \left(q\exp(6/2\sigma^2)\right. \\ &\,\,\,\,\left.+ (1-3q)\exp(2/2\sigma^2) - (1-2q)\right)\\
%&\leq \frac{q^2}{\sigma^2} + O(q^3/\sigma^2).
%\end{align*}

To bound the remaining terms, we first note that by standard calculus, we
get:
\begin{align*}
\forall z \leq 0: |\mu_0(z) - \mu_1(z)| &\leq -(z-1)\mu_0(z)/\sigma^2,\\
\forall z \geq 1: |\mu_0(z) - \mu_1(z)| &\leq z\mu_1(z)/\sigma^2,\\
\forall 0 \leq z \leq 1: |\mu_0(z) - \mu_1(z)|
&\leq \mu_0(z)(\exp(1/2\sigma^2) - 1) \\&\leq \mu_0(z)/\sigma^2.\\
\end{align*}

We can then write
\begin{align*}
\E_{z \sim \mu}&\left[\left(\frac{\mu_0(z) - \mu(z)}{\mu(z)}\right)^t\right]\\
&\leq \int_{-\infty}^0 \mu(z) \left|\left(\frac{\mu_0(z)
- \mu(z)}{\mu(z)}\right)^t\right| \ud z \\&\quad+ \int_{0}^1 \mu(z) \left|\left(\frac{\mu_0(z)
- \mu(z)}{\mu(z)}\right)^t\right| \ud z \\&\quad + \int_1^{\infty} \mu(z) \left|\left(\frac{\mu_0(z) - \mu(z)}{\mu(z)}\right)^t\right|\ud z.
\end{align*}
We consider these terms individually. We repeatedly make use of three
observations: (1) $\mu_0 - \mu = q(\mu_0 - \mu_1)$, (2) $\mu \geq
(1-q) \mu_0$, and (3) $\E_{\mu_{0}}[|z|^{t}] \leq \sigma^{t} (t-1)!!$. The first term can then be bounded by
\begin{align*}
\frac{q^t}{(1-q)^{t-1}\sigma^{2t}}&\int_{-\infty}^{0} \mu_0(z) |z-1|^t \ud z\\
&\leq \frac{(2q)^t (t-1)!! }{2(1-q)^{t-1}\sigma^t}.
\end{align*}
The second term is at most
\begin{align*}
\frac{q^t}{(1-q)^t} \int_{0}^1 \mu(z) &\left|\left(\frac{\mu_0(z)
- \mu_1(z)}{\mu_0(z)}\right)^t\right| \ud
  z \\&\leq \frac{q^t}{(1-q)^t}\int_0^{1} \mu(z) \frac{1}{\sigma^{2t}} \ud
  z\\
&\leq \frac{q^t}{(1-q)^t \sigma^{2t}}.
\end{align*}

 Similarly, the third term is at most
\begin{align*}
&\frac{q^t}{(1-q)^{t-1}\sigma^{2t}} \int_1^{\infty} \mu_{0}(z)
\left(\frac{z\mu_1(z)}{\mu_{0}(z)}\right)^t \ud z\\
&\;\;\;\;\leq \frac{q^t}{(1-q)^{t-1}\sigma^{2t}}\int_1^{\infty} \mu_{0}(z)
\exp((2tz - t)/2\sigma^2) z^t \ud z\\
&\;\;\;\;\leq \frac{q^t\exp((t^2-t)/2\sigma^2)}{(1-q)^{t-1}\sigma^{2t}}\int_{0}^\infty \mu_0(z-t)
z^t \ud z\\
&\;\;\;\;\leq \frac{(2q)^t \exp((t^2-t)/2\sigma^2) (\sigma^t(t-1)!! + t^t)}{2(1-q)^{t-1}\sigma^{2t}}.
\end{align*}
Under the assumptions on $q$, $\sigma$, and $\lambda$, it is easy to check that the three terms, and their sum, drop off geometrically fast in $t$ for $t>3$. Hence the binomial expansion (\ref{eqn:binomial}) is dominated by the $t=3$ term, which is $O(q^3\lambda^3/\sigma^3)$. The claim follows.
\end{proof}

To derive Theorem~\ref{thm:main}, we use the above moments bound along with the tail bound from Theorem~\ref{thm:property_supp}, optimizing over the choice of $\lambda$.
\setcounter{thm}{0}\begin{thm}
There exist constants $c_1$ and $c_2$ so that given the sampling probability $q=L/N$ and the number of steps $T$, for any $\eps < c_1 q^2T$,  Algorithm~\ref{alg:privsgd} is $(\eps,\delta)$-differentially private for any $\delta>0$ if we choose 
\begin{align*}
\sigma \geq c_2\frac{q \sqrt{T \log(1/\delta)}}{\eps}\,.
\end{align*}
\end{thm}
\begin{proof}
Assume for now that $\sigma, \lambda$ satisfy the conditions in Lemma~\ref{lem:sampled_gaussian_mgf}.  By Theorem~\ref{thm:property}.\ref{thm:partone} and Lemma~\ref{lem:sampled_gaussian_mgf}, the log moment of Algorithm~\ref{alg:privsgd} can be bounded as follows $\alpha(\lambda) \leq T q^2 \lambda^2/\sigma^2$. By Theorem~\ref{thm:property}, to guarantee Algorithm~\ref{alg:privsgd} to be $(\eps,\delta)$-differentially private, it suffices that
\begin{align*}
T q^2 \lambda^2/\sigma^2 &\leq \lambda \eps /2\,,\\
\exp(-\lambda\eps/2) &\leq \delta\,.
\end{align*}
In addition, we need
\begin{align*}
\lambda \leq \sigma^2 \log(1/q\sigma)\,.
\end{align*}
It is now easy to verify that when $\eps=c_1 q^2T$, we can satisfy all these conditions by setting
\begin{align*}
\sigma = c_2 \frac{q \sqrt{T \log(1/\delta)}}{\eps}\,
\end{align*}
for some explicit constants $c_1$ and $c_2$.
\end{proof}

\section{From differential privacy to moments bounds}

One can also translate a differential privacy guarantee into a moment bound.
\newtheorem{thm_appendix}{Theorem}[section]
\newtheorem{lem_appendix}[thm_appendix]{Lemma}
\begin{lem_appendix}
  \label{lem:pure_dp_to_logmgf}
  Let $\M$ be $\eps$-differentially private. Then for any $\lambda > 0$, $\M$ satisfies
\begin{align*}
\alpha_{\lambda} \leq \lambda\eps(e^{\eps}-1) + \lambda^2\eps^2e^{2\eps}/2.
\end{align*}
  \end{lem_appendix}
\begin{proof}
  Let $Z$ denote the random variable $c(\M(d))$. Then differential privacy implies that
  \begin{itemize}
      \item $\mu \eqdef \E[Z] \leq \eps(\exp(\eps) - 1)$.  
      \item $|Z| \leq \eps$, so that $|Z-\mu| \leq \eps \exp(\eps)$.
    \end{itemize}
  Then $\E[\exp(\lambda Z)] = \exp(\lambda\mu) \cdot \E[\exp (\lambda(Z-\mu))]$. Since $Z$ is in a bounded range $[-\eps\exp(\eps), \eps\exp(\eps)]$ and $f(x) = \exp(\lambda x)$ is convex, we can bound $f(x)$ by a linear interpolation between the values at the two endpoints of the range. Basic calculus then implies that
  \begin{align*}
    \E[f(Z)] \leq f(\E[Z]) \cdot \exp(\lambda^2 \eps^2 \exp(2\eps)/2),
    \end{align*}
 which concludes the proof.
 \end{proof}
\mycomment{Similarly,
\begin{lem}
  \label{lem:approx_dp_to_logmgf}
  Let $\M$ be $(\eps,\delta)$-DP. Then $\M$ is $(\lambda\eps(e^{\eps}-1) + \lambda^2\eps^2e^{2\eps}/2, \delta, \lambda)$-logmgf bounded.
  \end{lem}
}
Lemma~\ref{lem:pure_dp_to_logmgf} and Theorem~\ref{thm:property_supp} give a way of getting a composition
theorem for differentially private mechanisms, which is roughly equivalent
to unrolling the proof of the strong composition theorem
of~\cite{DRV10-boosting}. 
The power of the \logmgfa accountant comes from the
fact that, for many mechanisms of choice, directly bounding in the
\logmgfa gives a stronger guarantee than one would get by establishing
differential privacy and applying Lemma~\ref{lem:pure_dp_to_logmgf}.
\mycomment{
or~\ref{lem:approx_dp_to_logmgf}. 
A case in point is the Gaussian
noise mechanism, for which the following bound follows from Corollary
5 in~\cite{Mironov16}. 
\begin{lem}
  \label{lem:gaussian_logmgf}
  If $f$ has $\ell_2$-sensitivity $1$, then the mechanism $\M(d) = f(d) + \calN(0,\sigma^2 \Id)$  is  $(\lambda(\lambda+1)/2\sigma^2, \lambda)$-logmgf bounded. \ktnote{Check this bound.}
  \end{lem}
}

\section{Hyperparameter Search}
Here we state Theorem 10.2 from~\cite{GuptaLMRT10} that we use to
account for the cost of hyperparameter search.
\begin{thm_appendix}[Gupta et al.~\cite{GuptaLMRT10}]
\label{thm:hyperparameters}
Let $M$ be an $\eps$ -differentially private mechanism
such that for a query function $q$ with sensitivity 1,
and a parameter $Q$, it holds that $\Pr_{r \sim M(d)}[q(d, r) \geq Q] \geq p$ for
some $p \in (0, 1)$. Then for any $\delta  > 0$ and any $\eps' \in
(0,\frac 1 2)$, there is a mechanism $M'$ which satisfies the following
properties:
\begin{itemize}
\item $\Pr_{r \sim M'(d)}\left[q(d, r) \geq Q - \frac{4}{\eps'}\log
(\frac{1}{\eps'\delta p})\right] \geq 1 - \delta$.
\item $M'$ makes $(\frac{1}{\eps'\delta p})^2 \log
(\frac{1}{\eps'\delta p})$ calls to $M$.
\item $M'$ is $(\eps + 8 \eps')$-differentially private.
\end{itemize}
\end{thm_appendix}

Suppose that we have a differentially private mechanism $M_i$ for each of
$K$ choices of hyperparameters. Let $\tilde{M}$ be the mechanism that picks a random choice of
hyperparameters, and runs the corresponding $M_i$. Let $q(d,r) $ denote the
number of examples from the validation set the $r$ labels correctly,
and let $Q$ be a target accuracy. Assuming that one of the hyperparameter
settings gets accuracy at least $Q$, $\tilde{M}$ satisfies the
pre-conditions of the theorem for $p = \frac 1 K$. Then with high
probability, the mechanism implied by the theorem gets accuracy close
to $Q$. We remark that the proof of Theorem~\ref{thm:hyperparameters}
actually implies a stronger $\max(\eps, 8\eps')$-differential privacy
for the setting of interest here.

Putting in some numbers, for a target accuracy of $95\%$ on a validation set of size 10{,}000, we get $Q=9,500$. Thus, if, for instance, we allow $\eps' = 0.5$, and $\delta = 0.05$, we lose at most $1\%$ in accuracy as long as $100 > 8 \ln \frac {40}{p}$. This is satisfied as long as $p \geq \frac{1}{6700}$. In other words, one can try 6{,}700 different parameter settings at privacy cost $\eps=4$ for the validation set. In our experiments, we tried no more than a hundred settings, so that this bound is easily satisfied. In practice, as our graphs show, $p$ for our hyperparameter search is significantly larger than $\frac{1}{K}$, so that a slightly smaller $\eps'$ should suffice.
}{}
\end{document}